%% file: Arxiv_OptimalThompson_v2.tex
\documentclass{article}
\usepackage{amsmath}
\usepackage{amsfonts}
\usepackage{amssymb}
\usepackage{textcomp}
\usepackage{graphicx}
\usepackage{color}
\usepackage{bm}
\usepackage{amssymb,amsmath,graphicx,amscd}
\DeclareGraphicsExtensions{.jpg,.pdf,.mps,.eps,.png}

\newtheorem{theorem}{Theorem}

\newtheorem{lemma}{Lemma}

\newtheorem{proposition}{Proposition}
\newtheorem{remark}{Remark}
\newtheorem{definition}{Definition}

\newenvironment{proof}[1][Proof]{\noindent\textbf{#1.} }{\ \rule{0.5em}{0.5em}}

\newcommand{\cB}{\mathcal{B}}

\newcommand{\cI}{\mathcal{I}}
\newcommand{\cJ}{\mathcal{J}}

\newcommand{\cR}{\mathcal{R}}
\newcommand{\cS}{\mathcal{S}}

\newcommand{\bE}{\mathbb{E}}
\newcommand{\bP}{\mathbb{P}}

\newcommand{\Ind}{\mathbf{1}}

\newcommand{\ra}{\rightarrow}

\newcommand{\muhat}{\hat{\mu}}

\title{Thompson Sampling:\\ An Asymptotically Optimal Finite Time Analysis}
\author{Emilie Kaufmann, Nathaniel Korda and R\'emi Munos \\ \\ \small{Telecom Paristech UMR CNRS 5141 \& INRIA Lille - Nord Europe}}

\begin{document}

\maketitle

\begin{abstract}
\input{Abstract}
\end{abstract}

\section{Introduction}

\input{Introduction}

\section{Preliminaries}\label{preliminaries}

\input{ModelNotation}

\section{Finite Time Analysis}\label{analysis}

\subsection{Sketch of Analysis}

\input{SketchAnalysis}

\subsection{Proof of Theorem \ref{Thm2Arm}}

\input{Analysis}

\subsection{Proof of Proposition \ref{important:lemma}}\label{OptimalControl}

\input{Theorem1GeneralCase}

\subsection{Proof of Lemma \ref{calculs}}\label{Proofs}

\input{TechnicalLemma}

\section{Experiments}\label{experiments}

\input{Experiments}

\section{Discussion}

\input{Conclusion}

\bibliography{biblio}

\end{document}

%% file: Abstract.tex
The question of the optimality of Thompson Sampling for solving the stochastic multi-armed bandit problem had been open since 1933. 
In this paper we answer it positively for the case of Bernoulli rewards by providing the first finite-time analysis that matches 
the asymptotic rate given in the Lai and Robbins lower bound for the cumulative regret. The proof is accompanied by a numerical 
comparison with other optimal 
policies, experiments that have been lacking in the literature until now for the Bernoulli case.

%% file: Introduction.tex
In a stochastic bandit problem an agent is repeatedly 
asked to choose  one  action from an action set, each of which produces a reward drawn from an underlying, fixed, but unknown distribution 
associated with each action. 
Thus he must choose at each time whether to use the observations he has already gathered to gain the greatest immediate reward (exploitation) 
or whether to choose an action from which few observations have been made and risk immediate loss for greater knowledge and potential future gain 
(exploration).
In this paper we focus on stochastic bandits with Bernoulli rewards, initially proposed by Thompson in his paper of 1933 \cite{Thompson33} to model medical 
allocation problems. Thompson's paper also presented the first bandit algorithm, Thompson Sampling. This algorithm has received much attention in the recent literature, and in this paper
 we give the first theoretical proof of the  asymptotic optimality of this algorithm in the context of cumulative regret minimisation.  Furthermore we achieve this result by giving a finite time analysis for the algorithm.\\ \par
 
Associated with each action, $a$, is an unknown Bernoulli 
distribution $\cB\left(\mu_a\right)$, whose expectation is $\mu_a$. At each time $t$ the agent chooses to observe an action $A_t\in\{1,\dots,K\}$ and receives 
a reward $R_t$ drawn from the distribution $\cB\left(\mu_{A_t}\right)$. A policy, or bandit algorithm, is defined to be a (possibly randomised) method for choosing $A_t$
given the past history of observations and actions. The agent's goal is to minimize the expected cumulative regret of his policy, which is defined to be:
\begin{equation}\label{regretgap}
\cR(T):=T\mu^* - \bE\left[\sum_{t=1}^{T}R_t\right] = \sum_{a\in A}(\mu^* - \mu_a)\bE[N_{a,t}]
\end{equation}
where $\mu^*=\max_{a}\mu_a$ denotes the expectation of the best arm\footnote{The words arms and actions are used interchangably.}, or optimal action, and $N_{a,t}$ the number of draws of arm $a$ at the end of round $t$.
 Lai and Robbins proved in \cite{LaiRobbins85bandits} that all \emph{strongly consistent} policies 
(i.e. policies satisfying $\cR(t)=o(t^{\alpha})$ for all $\alpha\in(0,1)$) must satisfy, for any suboptimal arm $a$
\begin{align}
\liminf_{T\ra\infty}\frac{\bE[N_{a,T}]}{\ln T}\geq \frac{1}{K(\mu_a,\mu^*)}\label{LaiRobbins}
\end{align}
where $K(p,q)$ denotes the Kullback-Leibler divergence between $\cB\left(p\right)$ and $\cB\left(q\right)$:
\[
K(p,q):=p\ln\frac{p}{q}+(1-p)\ln\frac{1-p}{1-q}.
\]
Their result, which holds for more general classes of reward distributions, leads to the definition of 
\emph{asymptotically optimal} policies as policies that satisfy (\ref{LaiRobbins}) with equality.\\ \par

In the same paper \cite{LaiRobbins85bandits} Lai and Robbins were able to describe an asymptotically optimal policy, however no finite-time analysis was provided, 
nor was it an efficient policy to implement. The UCB1 algorithm by Auer et al. \cite{AuerEtAl02FiniteTime} was the first of a series of efficient policies, 
like UCB-V \cite{AudibertEtAlUCBV} or MOSS \cite{AudBubMOSS10}, for which good regret bounds in finite time were also provided.
These policies all use an upper confidence bound for the empirical mean of past rewards as an optimistic index for each arm, choosing at each time the action with the highest current index.
However, for each of these algorithms we only have the result that there exists two constants $K_1>2$ and $K>0$ such that for every suboptimal action $a$, with $\Delta_a=\mu^*-\mu_a$,
\begin{equation}\bE[N_{a,T}] \leq \frac{K_1}{\Delta_a^2}\ln(T) + K_2\label{NonOptRegret}.\end{equation}
This does not imply (\ref{LaiRobbins}) with equality since by the Pinsker inequality $2 K(\mu_a,\mu^*)> \Delta_a^2$. 
On the contrary, recently proposed index policies such as DMED \cite{HondaTakemura10DMED} and KL-UCB \cite{AOKLUCB,Remi:Odalric:KLUCB}, which use indices obtained from KL-based confidence 
regions, have been shown to be asymptotically optimal.\\ \par

Unlike most of this family of upper confidence bound algorithms that has been so successful, Thompson Sampling is a policy that uses ideas from Bayesian modelling and yet 
it solves the fundamentally frequentist problem of regret minimisation. Assume a uniform prior on each parameter $\mu_a$, let $\pi_{a,t}$ denote the posterior distribution for $\mu_a$ 
after the $t^{th}$ round of the algorithm. Let $\theta_{a,t}$ denote a sample from $\pi_{a,t}$; we sometimes refer to $\theta_{a,t}$ as a \emph{Thompson sample}. 
Thompson sampling is the policy which at time $t$ chooses to observe the action with the highest Thompson sample $\theta_{a,t}$, i.e. it chooses action $a$ with the 
probability that this action has the highest expected reward under the posterior distribution. \\ \par

Before Agrawal and Goyal's recent paper \cite{Agrawal:Goyal} Thompson Sampling had been investigated in \cite{GranmoBla} as the Bayesian Learning Automaton, and in \cite{MayKordaOBS} 
where an optimistic version was also proposed; however these papers only provided weak theoretical guarantees. In \cite{ChapelleLiEmpirical} extensive numerical 
experiments were carried out for Thompson Sampling beyond the scope of the Bernoulli bandit setting (to the Generalized Linear Bandit Model) but without any theoretical guarantee at all. 
Consequently the first finite-time analysis of Thompson Sampling in \cite{Agrawal:Goyal} was a major breakthrough, yet the upper bound for the regret that is shown in this paper scales like 
(\ref{NonOptRegret}) and the question of Thompson Sampling's asymptotic optimality was still open.
\\ \par
Meanwhile, there has been a resurgence of interest in Bayesian strategies for bandit problems (see \cite{AISTATS12} for a review of them). The Bayes-UCB algorithm, 
an upper confidence bound policy which uses an adaptive quantile of $\pi_{a,t}$ as an optimistic index, was the first Bayesian algorithm to be proved asymptotically 
optimal. In this paper we are able to show that the same is true for a randomised Bayesian algorithm, 
Thompson Sampling. Moreover we refer in our analysis to the Bayes-UCB index when introducing the deviation between a Thompson Sample and the corresponding posterior
quantile.

\paragraph{Contributions} We provide a finite-time regret bound for Thompson Sampling, that follows from (\ref{regretgap}) and from the result on the
expected number of suboptimal draws stated in Theorem \ref{Thm2Arm}. More precisely we show the following:

\begin{theorem}For every $\epsilon>0$ there exists a problem-dependent constant $C(\epsilon,\mu_1,\dots,\mu_K)$ such that the regret of Thompson Sampling satisfies:
\[
\cR(T)\leq(1+\epsilon)\sum_{a\in A:\mu_a\neq\mu^*}\frac{\Delta_a(\ln(T)+\ln\ln(T))}{K(\mu_a,\mu^*)}+ C(\epsilon,\mu_1,\dots,\mu_K).
\]
\end{theorem}
Besides this asymptotically optimal regret bound, we also provide the first numerical experiments that show Thompson Sampling outperforming the 
current best optimal policies like DMED, KL-UCB or Bayes-UCB. 
The rest of the paper is structured as follows. Section \ref{preliminaries} contains notations or results already used in
 \cite{Agrawal:Goyal}, \cite{AOKLUCB} or \cite{AISTATS12}
that are useful in our finite-time analysis given in Section \ref{analysis}. Numerical experiments are presented in Section \ref{experiments}.

%% file: ModelNotation.tex
We gather together here some useful preliminaries such as notations not already given in the introduction:
\begin{itemize}
\item For the rest of this paper, we assume action $1$ is the unique optimal action.
Without loss of generality\footnote{In Appendix A of \cite{Agrawal:Goyal} the authors show that adding a second optimal arm can only improve the regret performance of Thompson Sampling.}, we can assume that the parameter $\mu=(\mu_1,...,\mu_K)$ of the problem is such that $\mu_1>\mu_2\geq...\geq\mu_K$.

\item We shall denote by $S_{a,t}$ the number of successes observed from action $a$ by time $t$, and denote the empirical mean by:
\[
\muhat_{a,t}:=\frac{S_{a,t}}{N_{a,t}}.
\]

\item In the Bernoulli case, with a uniform prior on the parameters $\mu_a$ of the arms, the posterior on arm $a$ at time $t$ is explicitly 
$$\pi_{a,t} = \text{Beta}\left(S_{a,t}+1, N_{a,t}-S_{a,t} + 1\right).$$

\item Let $F_{a,b}^{\text{Beta}}$ denote the cdf of a Beta distribution with parameters $a$ and $b$ and $F_{j,\mu}^{\text{B}}$ (resp $f_{j,\mu}^{\text{B}}$) the cdf (resp pdf) 
of a Binomial distribution with parameters $j$ and $\mu$. We recall an important link between Beta and Binomial distribution which was used in both \cite{Agrawal:Goyal} and \cite{AISTATS12}:
$$F_{a,b}^{\text{Beta}}(y)=1 - F^{B}_{a+b-1,y}(a-1)$$
We use this `Beta-Binomial trick' at several stages of our analysis.

\item  We denote by $u_{a,t}$ (resp. $q_{a,t}$) the KL-UCB (resp. Bayes-UCB) index at time $t$, and define them by 
\begin{align*}
u_{a,t}:&=\underset{x > \frac{S_{a,t}}{N_{a,t}} }{\text{argmax}}\left\{K\left(\frac{S_{a,t}}{N_{a,t}},x\right) \leq \frac{\ln(t) +\ln(\ln(T))}{N_{a,t}}\right\}\\
q_{a,t}:&=Q\left(1-\frac{1}{t\ln(T)} , \pi_{a,t}\right).
\end{align*}
where $Q(\alpha,\pi)$ denotes the quantile of order $\alpha$ of the distribution $\pi$
\item  A special link between these two indices is shown in \cite{AISTATS12}: $q_{a,t}<u_{a,t}$.
\end{itemize}

%% file: SketchAnalysis.tex
Unlike Agrawal and Goyal's analysis, which is based on explicit computation of the expectation $\bE[N_{2,T}]$, we are more inspired by standard analyses of frequentist 
index policies. Such policies compute, for each arm $a$ at round $t$, an index, $l_{a,t}$, from the sequence of rewards observed from $a$ by time t, and choose $A_t=\text{argmax}_a l_{a,t}$. A standard analysis of such a policy aims 
to bound the number of draws of a suboptimal arm, 
$a$, by considering two possible events that might lead to a play of this arm:
\begin{itemize}
 \item the optimal arm (arm 1) is under-estimated, i.e. $l_{1,t} < \mu_1 $;
 \item the optimal arm is not under-estimated and the suboptimal arm a is drawn.
\end{itemize}
Taking these to be a good description of when the suboptimal arm is drawn leads to the decomposition
\[
\bE[N_{a,T}] \leq \sum_{t=1}^{T} \bP\left(l_{1,t} < \mu_1\right) + \sum_{t=1}^{T}\bP\left((l_{a,t} \geq l_{1,t}> \mu_1)\cap(A_t=a)\right)
\]

The analysis of an optimistic algorithm then proceeds by showing that the left term (the``under-estimation" term) is $o\left(\ln(T)\right)$ and the right term
is of the form $\frac{1}{K(\mu_a,\mu_1)}\ln(T) + o\left(\ln(T)\right)$ (or at worst $\frac{2}{\Delta_a^2}\ln(T) + o\left(\ln(T)\right)$ as in the analysis of UCB1). 
This style of argument works for example for the KL-UCB algorithm \cite{AOKLUCB} and also for the Bayesian optimistic algorithm Bayes-UCB \cite{AISTATS12}. \\ \par
However we cannot directly apply this approach to analyse Thompson Sampling, since the sample $\theta_{a,t}$ is not an optimistic estimate of $\mu_a$ based on an upper confidence bound. Indeed, even when $\pi_{1,t}$ is well concentrated and therefore close to a Gaussian distribution centred at $\mu_1$,
$\bP\left(\theta_{1,t} < \mu_1\right)$ is close to $\frac{1}{2}$ and the under-estimation term is not $o\left(\ln(T)\right)$. 
Hence we will not compare in our proof the sample $\theta_{a,t}$ to $\mu_a$, but to $\mu_a - \sqrt{6 \ln(t)/N_{a,t}}$ (if $N_{a,t}>0$) which is the lower bound of an UCB interval.
We set the convention that if $N_{a,t}=0$, $ \sqrt{6 \ln(t)/N_{a,t}}=\infty$.

As is observed in \cite{Agrawal:Goyal} the main difficulty in a regret analysis for Thompson Sampling is to control the number of draws of the optimal arm. We provide this control in the 
form of Proposition \ref{important:lemma} whose proof, given in Section \ref{OptimalControl}, explores in depth the randomised nature of Thompson Sampling. 

\begin{proposition}\label{important:lemma}There exists constants $b=b(\mu_1,\mu_2)\in(0,1)$ and $C_b=C_b(\mu_1,\mu_2)<\infty$ such that
\[
\sum_{t=1}^{\infty}\bP\left(N_{1,t} \leq  t^b\right) \leq C_b.
\]
\end{proposition}

This proposition tells us that the probability that the algorithm has seen only a small number draws on arm 1 is itself small. As a result we can reduce to analysing 
the behaviour of the algorithm once it has seen a reasonable number of draws on arm 1, and thus the posterior distribution is well concentrated. 

\begin{remark} In general, a result on the regret like $\bE[N_{1,t}] \geq t - K\ln(t)$ does not imply a deviation inequality for $N_{1,t}$ (see \cite{Salomon:Audibert11}). Proposition \ref{important:lemma}
is therefore a strong result, that enables us to adapt the standard analysis mentioned above.
\end{remark}

Using this result, the 
new decomposition finally yields the following theorem:

\begin{theorem}\label{Thm2Arm}
Let $\epsilon>0$. With $b$ as in Proposition \ref{important:lemma}, for every suboptimal arm $a$, there exist constants $D(\epsilon,\mu_1,\mu_a)$, $N(b,\epsilon,\mu_1,\mu_a)$ and $N_0(b)$ such that:
\begin{align*}
\bE[N_{a,T}] \leq (1+\epsilon)\frac{\ln(T) + \ln\ln (T) }{K(\mu_a,\mu_1)} 
			+ D(\epsilon,\mu_1,\mu_a) + N(b,\epsilon,\mu_1,\mu_a)+N_0(b) + 5 + 2C_b .
\end{align*}
\end{theorem}
The constants will be made more explicit in the proofs of Proposition \ref{important:lemma} and Theorem \ref{Thm2Arm}. The fact that Theorem \ref{Thm2Arm} holds for every $\epsilon>0$
gives us the asymptotic optimality of Thompson Sampling.

%% file: Analysis.tex
\paragraph{Step 1: Decomposition}

First we recall the modified decomposition mentioned above:
\begin{align*}
\bE[N_{a,T}] \leq &\sum_{t=1}^{T} \bP\left(\theta_{1,t} \leq \mu_1 - \sqrt{\frac{6 \ln(t)}{N_{1,t}}}\right)
	+ \sum_{t=1}^{T} \bP\left(\theta_{a,t} > \mu_1 - \sqrt{\frac{6 \ln(t)}{N_{1,t}}}, A_t = a\right) \\
\leq & \sum_{t=1}^{T} \bP\left(\theta_{1,t} \leq \mu_1 - \sqrt{\frac{6 \ln(t)}{N_{1,t}}}\right)\\
	&+ \sum_{t=1}^{T} \bP\left(\theta_{a,t} > \mu_1 - \sqrt{\frac{6 \ln(t)}{N_{1,t}}}, A_t = a, \theta_{a,t} < q_{a,t}\right) 
	+ \sum_{t=1}^{T}\bP\left(\theta_{a,t} > q_{a,t}\right) 
\end{align*}
The sample $\theta_{a,t}$ is not very likely to exceed the quantile of the posterior distribution $q_{a,t}$ we introduced:
\[
\sum_{t=1}^{T}\bP\left(\theta_{a,t} > q_{a,t}\right) \leq \sum_{t=1}^{T}\frac{1}{t\ln(T)} \leq \frac{1+\ln(T)}{\ln(T)}\leq 2
\]
where this last inequality follows for $T\geq e$. So finally, using that $u_{a,t}\geq q_{a,t}$,
\begin{equation}
\bE[N_{a,t}] \leq   \underbrace{\sum_{t=1}^{T} \bP\left(\theta_{1,t} \leq \mu_1 - \sqrt{\frac{6 \ln(t)}{N_{1,t}}}\right)}_{A} 
	 + \underbrace{\sum_{t=1}^{T} \bP\left(u_{a,t} > \mu_1 - \sqrt{\frac{6 \ln(t)}{N_{1,t}}}, A_t = a\right)}_{B}
	+ 2\label{MainThmDecomp}
\end{equation}

\paragraph{Step 2: Bounding term A}
Dealing with term $A$ boils down to showing a new self-normalized inequality adapted to the randomisation present in each round of the Thompson algorithm. 
\begin{lemma} \label{LemmaMainThm1} There exists some deterministic constant  $N_0(b)$ such that 
\[
\sum_{t=1}^{\infty} \bP\left(\theta_{1,t} \leq \mu_1 - \sqrt{\frac{6 \ln(t)}{N_{1,t}}}\right)
	\leq N_0(b) + 3 +  C_b < \infty
\]
with $C_b$ defined as in Proposition \ref{important:lemma}.
\end{lemma}
\begin{proof} Let $(U_t)$ denote a sequence of i.i.d. uniform random variables, and let $\Sigma_{1,s}$ be the sum of the first $s$ rewards from arm 1. In the following, we make the first use of the link between Beta and Binomial distributions:
\begin{align*}
\bP&\left(\theta_{1,t} \leq \mu_1 - \sqrt{\frac{6 \ln(t)}{N_{1,t}}}\right)  = 
\bP\left(U_t \leq F^{\text{Beta}}_{S_{1,t}+1,N_{1,t}-S_{1,t}+1}\left(\mu_1 - \sqrt{\frac{6 \ln(t)}{N_{1,t}}}\right)\right) \\
& = \bP\left(\left(U_t \leq 1- F^{\text{B}}_{N_{1,t}+1,\mu_1 - \sqrt{\frac{6 \ln(t)}{N_{1,t}}}}\left(S_{1,t}\right)\right)\cap\left(N_{1,t} \geq t^b\right)\right) + \bP\left(N_{1,t} \leq t^b\right)\\
& = \bP\left(\left(F^{\text{B}}_{N_{1,t}+1,\mu_1 - \sqrt{\frac{6 \ln(t)}{N_{1,t}}}}\left(S_{1,t}\right)\leq U_t\right)\cap\left(N_{1,t} \geq t^b\right)\right) + \bP\left(N_{1,t} \leq t^b\right) \\
& \leq \bP\left(\exists s \in \{t^b...t\}: F^{\text{B}}_{s+1,\mu_1 - \sqrt{\frac{6 \ln(t)}{s}}}\left(\Sigma_{1,s}\right)\leq U_t\right) + \bP\left(N_{1,t} \leq t^b\right) \\
& =   \sum_{s=\lceil t^b \rceil}^{t}\bP\left(\Sigma_{1,s}\leq (F^{\text{B}})^{-1}_{s+1,\mu_1 - \sqrt{\frac{6 \ln(t)}{s}}}\left(U_t\right)\right) + \bP\left(N_{1,t} \leq t^b\right)
\end{align*}
The first term in the final line of this display now deals only with Binomial random variables with large numbers of trials (greater than $t^b$), and so we can draw on standard concentration techniques to bound this term. Proposition \ref{important:lemma} takes care of the second term.

Note that $(F^{\text{B}})^{-1}_{s+1,\mu_1 - \sqrt{6 \ln(t)/s}}\left(U_t\right)\sim\text{Bin}\left(s+1,\mu_1 - \sqrt{6 \ln(t)/s}\right)$
and is independent from $\Sigma_{1,s}\sim\text{Bin}\left(s,\mu_1\right)$. For each $s$, we define two i.i.d. sequences of Bernoulli random variables:
\[
(X_{1,l})_l \sim \mathcal{B}\left(\mu_1- \sqrt{\frac{6 \ln(t)}{s}}\right)\text{ and }(X_{2,l})_l \sim \mathcal{B}\left(\mu_1\right),
\]
and we let $Z_l:=X_{2,l}-X_{1,l}$, another i.i.d. sequence, with mean $\sqrt{\frac{6\ln(t)}{s}}$. Using these notations, 
$$
\bP\left(\Sigma_{1,s}\leq (F^{\text{B}})^{-1}_{s+1,\mu_1 - \sqrt{\frac{6 \ln(t)}{s}}}\left(U_t\right)\right) \leq \bP\left(\sum_{l=1}^{s} Z_l \leq 1\right) = \bP\left(\sum_{l=1}^{s} \left(Z_l - \sqrt{\frac{6\ln(t)}{s}}\right) \leq -\left(\sqrt{6s\ln(t)} -1\right)\right).
$$
Let $N_0(b)$ be such that if $t\geq N_0(b)$,  $\sqrt{6t^b\ln(t)} -1 >  \sqrt{5t^b\ln(t)}$. For $t\geq N_0(b)$, we can apply
Hoeffding's inequality to the bounded martingale difference sequence $Z'_l=Z_l - \sqrt{6\ln(t)/s}$ to get
$$\bP\left(\Sigma_{1,s}< (F^{\text{B}})^{-1}_{s+1,\mu_1 - \sqrt{\frac{6 \ln(t)}{s}}}\left(U_t\right)\right) \leq \exp\left(-2\frac{(\sqrt{5s\ln(t)})^2}{4s}\right) = e^{-\frac{5}{2}\ln(t)} = \frac{1}{t^{\frac{5}{2}}}.$$
We conclude that
$$\sum_{t=1}^\infty\bP\left(\theta_1(t) < \mu_1 - \sqrt{\frac{6 \ln(t)}{N_{1,t}}}\right) \leq N_0(b) + \sum_{t=1}^{\infty} \frac{1}{t^{\frac{3}{2}}} + C_b \leq N_0(b) + 3 + C_b.$$
\end{proof}

\paragraph{Step 3: Bounding Term B}
We specifically show that:
\begin{lemma}\label{LemmaMainThm2}
For all $a=2,\dots,K$, for any $\epsilon>0$ there exist $N(b,\epsilon,\mu_1,\mu_a),D(\epsilon,\mu_1,\mu_a)>0$ such that for all $T>N(b,\epsilon,\mu_1,\mu_a)$
\[
(B)\leq (1+\epsilon)\frac{\ln(T) + \ln\ln (T) }{K(\mu_a,\mu_1)} + D(\epsilon,\mu_1,\mu_a).
\]
\end{lemma}
\begin{proof}
First rewrite term $B$ so that we can apply Proposition \ref{important:lemma}: 
\begin{align*}
(B)& \leq  \sum_{t=1}^{T} \bP\left(u_{a,t} > \mu_1 - \sqrt{\frac{6 \ln(t)}{N_{1,t}}}, A_t = a, N_{1,t} \geq t^b\right) 
					+\sum_{t=1}^{T}\bP\left(N_{1,t} \leq  t^b\right)\\
   & \leq \sum_{t=1}^{T} \bP\left(u_{a,t} > \mu_1 - \sqrt{\frac{6 \ln(t)}{t^b}}, A_t = a \right) + C_b
\end{align*}
For ease of notation we introduce
\begin{align*}
K^+(x,y):=K(x,y)\Ind_{(x\leq y)},\  f_T(t):=\ln(t)+\ln(\ln(T))\\
\beta_t = \sqrt{\frac{6 \ln(t)}{t^b}},\text{ and } K_{T,a}(\epsilon) = (1+\epsilon)\frac{\ln(T) + \ln\ln (T) }{K(\mu_a,\mu_1)}.
\end{align*}

Now 
\[
\left(u_{a,t} \geq \alpha\right) = \left(N_{2,t}K^+(\hat{\mu}_{2,N_{2,t}},\alpha)
											\leq f_T(t)\right)
\]
and so summing over the values of $N_{2,t}$ and inverting the sums we get
\begin{align*}
\sum_{t=1}^{T} \bP\left(u_{a,t} > \mu_1 - \beta_t, A_t = a \right)   
  =& \bE\left[\sum_{s=1}^{\lfloor K_{T,a} \rfloor}\sum_{t=s}^{T} \Ind_{\left(sK^+\left(\muhat_{a,s},\mu_1 - \beta_t\right) \leq f_T(t)\right)}\Ind_{(A_t = a , N_{2,t}=s)}\right] \\
 & + \bE\left[\sum_{s=\lfloor K_{T,a} \rfloor+1}^{T}\sum_{t=s}^{T} \Ind_{\left(sK^+\left(\muhat_{a,s},\mu_1 - \beta_t\right) \leq f_T(t)\right)}\Ind_{(A_t = a , N_{2,t}=s)}\right]. 
\end{align*}
As $y \mapsto K^+\left(\muhat_{a,s},y \right)$ is increasing and $t \mapsto \beta_t$ is decreasing for $t \geq e^{1/b}$, for $T$ such that 
\begin{align}
K_{T,a}(\epsilon) \geq e^{1/b}\label{Ktime},
\end{align}
we have that if $t \geq K_{T,a}(\epsilon)$,
$$\Ind_{\left(sK^+\left(\muhat_{a,s},\mu_1 - \beta_t\right) \leq f_T(t)\right)}\leq \Ind_{\left(sK^+\left(\muhat_{a,s},\mu_1 - \beta_{K_{T,a}}\right) \leq f_T(t)\right)}
\leq \Ind_{\left(sK^+\left(\muhat_{a,s},\mu_1 - \beta_{K_{T,a}}\right) \leq f_T(T)\right)}.$$
and therefore,
\begin{align*}
\sum_{t=1}^{T} \bP\left(u_{a,t} > \mu_1 - \beta_t, A_t = a \right)  \leq \bE&\left[\sum_{s=1}^{\lfloor K_{T,a} \rfloor}\sum_{t=s}^{T}\Ind_{(A_t = a , N_{2,t}=s)}\right] \\
 +  \bE&\left[\sum_{s=\lfloor K_{T,a} \rfloor+1}^{T} \Ind_{\left(sK^+\left(\muhat_{a,s},\mu_1 - \beta_{K_{T,a}}\right) \leq f_T(T)\right)}\sum_{t=s}^{T}\Ind_{(A_t = a , N_{2,t}=s)}\right].
\end{align*}
Given that $\sum_{t=s}^{T} \Ind_{(A_t = a , N_{2,t}=s)} \leq 1$ for all $s$, the first term is upper bounded by $K_{a,T}$, whereas the second is upper bounded by
$$
\bE\left[\sum_{s=\lfloor K_{T,a} \rfloor+1}^{T}\Ind_{\left(K_{T,a} K^+\left(\muhat_{a,s},\mu_1 - \beta_{K_{T,a}}\right) \leq f_T(T)\right)}\right]
$$
So, for $T$ satisfying (\ref{Ktime}),
\[
(B) \leq K_{T,a} + \sum_{\lfloor K_{T,a} \rfloor +1}\bP\left(K^+\left(\muhat_{a,s},\mu_1 -\beta_{K_{T,a}} \right)
							\leq \frac{K(\mu_a,\mu_1)}{1+\epsilon}\right).
\]
Using the convexity of $K^+\left(\muhat_{a,s},.\right)$, we can show that
$$
K^+(\hat{\mu}_{a,s},\mu_1) \leq K^+(\hat{\mu}_{a,s},\mu_1-\beta_{K_{a,T}}) + \frac{2}{\mu_1(1-\mu_1)}\beta_{K_{a,T}}. 
$$
If $K^+\left(\muhat_{a,s},\mu_1 -\beta_{K_{T,a}} \right) \leq K(\mu_a,\mu_1)/(1+\epsilon)$, then
\begin{align}
K^+(\hat{\mu}_{a,s},\mu_1) &\leq \frac{K(\mu_a,\mu_1)}{1+\epsilon} + \frac{2}{\mu_1(1-\mu_1)}\beta_{K_{a,T}}\nonumber \\
&\leq  \frac{K(\mu_a,\mu_1)}{1+\epsilon/2}\label{Kineq}
\end{align}
where the last inequality (\ref{Kineq}) holds for large enough $T$. There exists a deterministic constant $N=N(b,\epsilon,\mu_1,\mu_a)$ such that for all $T\geq N$ both (\ref{Ktime}) and (\ref{Kineq}) are satisfied. Hence, for all $T\geq N$
\[
(B) \leq K_{T,a} + \sum_{\lfloor K_{T,a} \rfloor +1}\bP\left(K^+\left(\muhat_{a,s},\mu_1\right)
							\leq \frac{K(\mu_a,\mu_1)}{1+\frac{\epsilon}{2}}\right).
\]
Since this last sum is bounded above explicitly by some constant $D(\epsilon,\mu_1,\mu_a)$ in \cite{Remi:Odalric:KLUCB} we have proved the lemma. To be explicit, 
$D(\epsilon,\mu_1,\mu_a)= \frac{(1+\epsilon/2)^2}{\epsilon^2\left(\min\left(\mu_a(1-\mu_a);\mu_1(1-\mu_1)\right)\right)^2}.$
\end{proof}

\paragraph{Conclusion:}The result now follows from Lemmas \ref{LemmaMainThm1}, \ref{LemmaMainThm2} and inequality (\ref{MainThmDecomp}).

%% file: Theorem1GeneralCase.tex
Since we focus on the number of draws of the optimal arm, let $\tau_{j}$ be the occurrence of the $j^{th}$ play of the optimal arm (with $\tau_0:=0$). 
Let $\xi_{j}:=(\tau_{j+1}-1)-\tau_{j}$: this random variable measures the number of time steps between the $j^{th}$ and the $(j+1)^{th}$ play of 
the optimal arm, and so $\sum_{a=2}^{K}N_{a,t}= \sum_{j=0}^{N_{1,t}}\xi_{j}$. \\ \\
 For each 
suboptimal arm, a relevant quantity is $C_a = \frac{32}{(\mu_1 - \mu_a)^2}$  and let $C=\max_{a\neq 1}C_a = \frac{32}{(\mu_1 - \mu_2)^2}$. We also
introduce $\delta_a=\frac{\mu_1-\mu_a}{2}$ and let $\delta=\delta_2$. 

\paragraph{Step 1: Initial Decomposition of Summands}
First we use a union bound on the summands to extract the tails of the random variables $\xi_{j}$:
\begin{align}
 \bP(N_{1,t} \leq t^b) & = \bP\left(\sum_{a=2}^{K}N_{2,t} \geq t - t^{b}\right) \nonumber\\
	& \leq \bP\left(\exists j\in \left\{ 0,.., \lfloor t^{b} \rfloor\right\} : \xi_{j} \geq t^{1-b} - 1\right)\nonumber \\
	& \leq \sum_{j=0}^{\lfloor t^{b} \rfloor} \bP(\xi_{j} \geq t^{1-b} - 1)\label{DecompInitial}
\end{align}
This means that there exists a time range of length $t^{1-b}-1$ during which only suboptimal arms are played. 
In the case of two arms this implies that the (unique) suboptimal arm is played $\lceil\frac{t^{1-b}-1}{2}\rceil$ times 
during the first half of this time range. Thus its posterior becomes well concentrated around its mean with high probability, 
and we can use this fact to show that the probability the suboptimal action is chosen a further $\lceil\frac{t^{1-b}-1}{2}\rceil$ times in a row is very small.

In order to generalise this approach we introduce a notion of a \emph{saturated}, suboptimal action:
\begin{definition} Let $t$ be fixed.
For any $a\neq1$, an action $a$ is said to be \emph{saturated} at time $s$ if it has been chosen at least $C_a\ln(t)$ times.
That is $N_{a,s} \geq C_a \ln(t)$. We shall say that it is \emph{unsaturated} otherwise. Furthermore at any time we call a choice of an unsaturated, 
suboptimal action an \emph{interruption}.
\end{definition}

We want to study the event $E_j=\{\xi_{j}\geq t^{1-b}-1\}$. We introduce the interval $\cI_j= \{\tau_j,\tau_j+\lceil t^{1-b}-1\rceil\}$ (included in $\{\tau_j,\tau_{j+1}\}$
on $E_j$) and begin by decomposing it into $K$ subintervals:
\[
\cI_{j,l}:=\left\{\tau_{j} + \left\lceil\frac{(l-1)(t^{1-b}-1)}{K}\right\rceil, \tau_{j} + \left\lceil\frac{l(t^{1-b}-1)}{K}\right\rceil\right\},\ l=1,\dots,K.
\]
Now for each interval $\cI_{j,l}$, we introduce:
\begin{itemize}
 \item $F_{j,l}$: the event that by the end of the interval $\cI_{j,l}$ at least $l$ suboptimal actions are saturated;
 \item $n_{j,l}$: the number of interruptions during this interval.
\end{itemize}
We use the following decomposition to bound the probability of the event $E_j$ :
\begin{equation}
\bP(E_j)=\bP(E_j\cap F_{j,K-1})+\bP(E_j\cap F_{j,K-1}^c)\label{decomp}
\end{equation}

To bound both probabilities, we will need the fact,  stated in Lemma \ref{calculs}, that the probability of $\theta_{1,s}$  being smaller than $\mu_2+\delta$ during a long subinterval of $\cI_j$ is small. This follows from 
the fact that the posterior on the optimal arm is always $\text{Beta}(S_{1,\tau_j}+1,j-S_{1,\tau_j} +1)$ on $\cI_j$: hence, 
when conditioned on $S_{1,\tau_j}$, $\theta_{1,s}$ is an i.i.d. sequence with non-zero support above $\mu_2+\delta$, and thus is unlikely to remain below $\mu_2+\delta$ for a long time period.
This idea is also an important tool in the analysis of Thompson Sampling in \cite{Agrawal:Goyal}.
\begin{lemma}
\label{calculs} $\exists \lambda_0=\lambda_0(\mu_1,\mu_2)>1$ such that for $\lambda \in ]1,\lambda_0[$, for every (random) interval $\cJ$ included in $\cI_{j}$ and for every positive
function $f$, one has
\begin{align*}
\bP\left(\{\forall s\in \cJ, \theta_{1,s} \leq \mu_2 + \delta\}\cap \{|\cJ| \geq f(t)\}\right)
\leq (\alpha_{\mu_1,\mu_2})^{f(t)} +  
C_{\lambda_0,\lambda}\frac{1}{f(t)^\lambda}e^{-jd_{\lambda,\mu_1,\mu_2}}
\end{align*}
where $C_{\lambda,\mu_1,\mu_2},d_{\lambda,\mu_1,\mu_2}>0$, and $\alpha_{\mu_1,\mu_2}=(1/2)^{1-\mu_2-\delta}$.
\end{lemma}
The proof of this important lemma will be postponed to section \ref{Proofs} and all the constants are explicitly defined there. Another keypoint in the proof is the fact that a sample from a 
saturated suboptimal arm cannot fall too far from its true mean. The following 
lemma is easily adapted from Lemma 2 in \cite{Agrawal:Goyal}.
\begin{lemma}\label{LemmaA}
\[
\bP\left(\exists s\leq t, \exists a \neq 1 : \theta_{a,s}>\mu_a+\delta_a,N_{a,s}>C_a \ln(t)\right)\leq\frac{2(K-1)}{t^2}.
\]
\end{lemma}

\paragraph{Step 2: Bounding  $\bm{\bP(E_j\cap  F_{j,K-1})}$} On the event $E_j\cap F_{j,K-1}$, only saturated suboptimal arms are drawn
on the interval $\cI_{j,K}$. 
Using the concentration results for samples of these arms in Lemma \ref{LemmaA}, we get
\begin{align*}
\bP(E_j\cap F_{j,K-1}) \leq& \bP(\left\{\exists s\in\cI_{j,K}, a \neq 1 : \theta_{a,s}>\mu_a+\delta\right\} \cap E_j \cap F_{j,K-1})\\
&+ \bP(\left\{\forall s\in\cI_{j,K},a\neq 1:\theta_{a,s}\leq\mu_a+\delta_a\right\}\cap E_j\cap F_{j,K-1})\\
\leq &  \bP(\exists s\leq t,a\neq 1 : \theta_{a,s}>\mu_a + \delta_a, N_{a,t}> C_a \ln(t))\\
&+ \bP(\left\{\forall s\in\cI_{j,K},a\neq 1:\theta_{a,s}\leq\mu_2+\delta\right\}\cap E_j\cap F_{j,K-1})\\
\leq& \frac{2(K-1)}{t^2}+\bP(\theta_{1,s}\leq\mu_2+\delta,\forall s\in\cI_{j,K}).
\end{align*}
The last inequality comes from the fact that if arm 1 is not drawn, the sample $\theta_{1,s}$ must be smaller than some sample $\theta_{a,s}$ and therefore smaller than $\mu_2+\delta$.
Since $\cI_{j,K}$ is an interval in $\cI_j$ of size $\left\lceil\frac{t^{1-b}-1}{K}\right\rceil$ we get using Lemma \ref{calculs}, for some fixed $\lambda\in ]1,\lambda_0[$,
\begin{align}
\bP (\theta_{1,s}\leq\mu_2+\delta,\forall s\in\cI_{j,K}\}) \leq \left(\alpha_{\mu_1,\mu_2}\right)^{\frac{t^{1-b}-1}{K}} + C_{\lambda,\mu_1,\mu_2}\left(\frac{t^{1-b}-1}{K}\right)^{-\lambda} e^{-jd_{\lambda,\mu_1,\mu_2}}=:g(\mu_1,\mu_2,b,j,t).
\label{FirstChoice}
\end{align}
Hence we have show that
\begin{equation}
\bP(E_j\cap  F_{j,K-1})\leq \frac{2(K-1)}{t^2} + g(\mu_1,\mu_2,b,j,t),\label{LeftTerm}
\end{equation}
and choosing $b$ such that $b < 1 -\frac{1}{\lambda}$, the following hypothesis on $g$ holds:
$$\sum_{t\geq 1}\sum_{j\leq t^b} g(\mu_1,\mu_2,b,j,t) < +\infty.$$

\paragraph{Step 3: Bounding $\bm{\bP(E_j\cap F_{j,K-1}^c}$)}We show through an induction that for all $2\leq l \leq K$, if $t$ is larger than some deterministic constant
$N_{\mu_1,\mu_2,b}$ specified in the base case,
$$
\bP(E_j\cap F_{j,l-1}^c)\leq (l-2)\left(\frac{2(K-1)}{t^2}+f(\mu_1,\mu_2,b,j,t)\right)
$$
for some function $f$ such that $\sum_{t\geq 1}\sum_{1\leq j\leq t^b}f(\mu_1,\mu_2,b,j,t)<\infty$. For $l=K$ we get
\begin{equation}
\bP(E_j\cap F_{j,K-1}^c)\leq (K-2)\left(\frac{2(K-1)}{t^2}+f(\mu_1,\mu_2,b,j,t)\right).\label{RightTerm}
\end{equation}

\paragraph{Step 4: The Base Case of the induction} Note that on the event $E_j$ only suboptimal arms are played during $\cI_{j,1}$. 
Hence at least one suboptimal arm must be played $\lceil\frac{t^{1-b}-1}{K^2}\rceil$ times.

There exists some deterministic constant $N_{\mu_1,\mu_2,b}$ such that for $t\geq N_{\mu_1,\mu_2,b}$, $\lceil\frac{t^{1-b}-1}{K^2}\rceil \geq C\ln(t)$
(the constant depends only on $\mu_1$ and $\mu_2$ because $C=C_2$). So when $t\geq N_{\mu_1,\mu_2,b}$, at least  one  suboptimal arm must be
saturated by the end of $\cI_{j,1}$. Hence, for $t\geq N_{\mu_1,\mu_2,b}$
\[
\bP(E_j\cap F_{j,1}^c) = 0.
\]
This concludes the base case.

\paragraph{Step 5: The Induction} As an inductive hypothesis we assume that for some $2\leq l \leq K-1$ if $t \geq N_{\mu_1,\mu_2,b}$ then
\[
\bP(E_j\cap F_{j,l-1}^c)\leq (l-2)\left(\frac{2(K-1)}{t^2}+f(\mu_1,\mu_2,b,j,t)\right).
\]
Then, making use of the inductive hypothesis,
\begin{align*}
\bP(E_j\cap F_{j,l}^c)&\leq\bP(E_j\cap F_{j,l-1}^c)+\bP(E_j\cap F_{j,l}^c\cap F_{j,l-1})\\	
	&\leq (l-2)\left(\frac{2(K-1)}{t^2}+f(\mu_1,\mu_2,b,j,t)\right)+\bP(E_j\cap F_{j,l}^c\cap F_{j,l-1}).
\end{align*}
To complete the induction we therefore need to show that:
\begin{equation}
\bP(E_j\cap F_{j,l}^c\cap F_{j,l-1})\leq \frac{2(K-1)}{t^2}+f(\mu_1,\mu_2,b,j,t).\label{goalinduction}
\end{equation}

On the event $(E_j\cap F_{j,l}^c\cap F_{j,l-1})$, there are exactly $l-1$ saturated arms at the beginning of interval $\cI_{j,l}$ and no new arm is saturated during this interval.
As a result there cannot be more than $K C \ln(t)$ interruptions during this interval, and so we have
$$\bP(E_j\cap F_{j,l}^c\cap F_{j,l-1})\leq \bP(E_j\cap F_{j,l-1}\cap \{n_{j,l} \leq K C \ln(t)\}).$$

Let $\cS_{l}$ denote the set of saturated arms at the end of $\cI_{j,l}$ and introduce the following decomposition:
\begin{align*}
\bP(E_j\cap F_{j,l-1}&\cap \{n_{j,l} \leq K C \ln(t)\})\nonumber\\
\leq& \underbrace{\bP(\{\exists s\in\cI_{j,l},a\in \cS_{l-1}:\theta_{a,s}>\mu_a+\delta_a\}\cap E_j\cap F_{j,l-1})}_{A}\nonumber\\
&+ \underbrace{\bP(\{\forall s\in\cI_{j,l},a\in \cS_{l-1}:\theta_{a,s}\leq\mu_a+\delta_a\}\cap E_j\cap F_{j,l-1}\cap  \{n_{j,l} \leq K C \ln(t)\})}_{B}.
\end{align*}
Clearly, using Lemma \ref{LemmaA}:
$$(A) \leq \bP\left(\exists s\leq t, \exists a \neq 1  : \theta_{a,s}>\mu_a+\delta_a,N_{a,s}>C_a \ln(t)\right)\leq\frac{2(K-1)}{t^2}.$$
To deal with term (B), we introduce for $k$ in $\{0,\dots,n_{j,l}-1\}$  the random intervals $\cJ_k$ as the time range between the $k^{th}$ and 
$(k+1)^{st}$ interruption in $\cI_{j,l}$. For $k\geq n_{j,l}$ we set $\cJ_k=\varnothing$.
Note that on the event in the probability (B) there is a subinterval of $\cI_{j,l}$ of length $\left\lceil\frac{t^{1-b}-1}{CK^2\ln(t)}\right\rceil$ during which there are no interruptions. Moreover on this subinterval of $\cI_{j,l}$, for all $a\neq 1$, $\theta_{a,s}\leq\mu_2+\delta_2$. (This holds for unsaturated arms as well as for saturated arms since their samples are smaller than the maximum sample of a saturated arm.)
Therefore,
\begin{align}
(B)&\leq \bP\left(\left\{\exists k \in \{0,...,n_{j,l}\} : |\cJ_k|\geq \frac{t^{1-b}-1}{CK^2\ln(t)}\right\}\cap \left\{\forall s\in\cI_{j,l},  a\in \cS_{l-1} 1:\theta_{a,s}\leq\mu_2+\delta\right\}\cap E_j\cap F_{j,l-1}\right) \nonumber \\
& \leq \sum_{k=1}^{KC\ln(t)}\bP\left(\left\{|\cJ_k|\geq \frac{t^{1-b}-1}{CK^2\ln(t)}\right\}\cap \left\{\forall s\in\cJ_k,  a\neq 1:\theta_{a,s}\leq\mu_2+\delta\right\}\cap E_j\right) \nonumber\\
& \leq \sum_{k=1}^{KC\ln(t)}\bP\left(\left\{|\cJ_k|\geq \frac{t^{1-b}-1}{CK^2\ln(t)}\right\}\cap \left\{\forall s \in \cJ_k, \ \theta_{1,s}  \leq \mu_2 + \delta\right\}\right) \label{almostgeom}
\end{align}
Now, we have to bound the probability that $\theta_{1,s}  \leq \mu_2 + \delta$ for all $s$ in an interval of size $\frac{t^{1-b}-1}{CK^2\ln(t)}$ included in $\cI_j$. So we apply Lemma \ref{calculs} to get:
$$
(B) \leq C K \ln(t) \left(\alpha_{\mu_1,\mu_2}\right)^{\frac{t^{1-b}-1}{C K^2\ln(t)}} +  
C_{\lambda,\mu_1,\mu_2}\frac{CK\ln(t)}{\left(\frac{t^{1-b}-1}{C K^2\ln(t)}\right)^\lambda}e^{-jd_{\lambda,\mu_1,\mu_2}}:=f(\mu_1,\mu_2,b,j,t).
$$
Choosing the same $b$ as in (\ref{FirstChoice}), we get that $\sum_{t\geq 1}\sum_{1\leq j \leq t^b} f(\mu_1,\mu_2,b,j,t) < +\infty$. It follows that for this value of $b$, (\ref{goalinduction}) holds and the induction is complete.

\paragraph{Step 8: Conclusion} Let $b$ be the constant chosen in Step 2. From the decomposition (\ref{decomp}) and the two upper bounds (\ref{LeftTerm}) and (\ref{RightTerm}), we get, for $t\geq N_{\mu_1,\mu_2,b}$: 
$$\bP(E_j) \leq (K-2)\left(\frac{2(K-1)}{t^2}+f(\mu_1,\mu_2,b,j,t))\right)+\frac{2(K-1)}{t^2}+g(\mu_1,\mu_2,b,j,t).$$
Recalling (\ref{DecompInitial}), summing over the possible values of $j$ and $t$ we obtain:
$$
 \sum_{t\geq 1} \bP(N_{1,t} \leq t^b) \leq N_{\mu_1,\mu_2,b} + 2(K-1)^2\sum_{t\geq 1}\frac{1}{t^{2-b}} + \sum_{t\geq 1}\sum_{j=1}^{t^b}[Kf(\mu_1,\mu_2,b,j,t)+g(\mu_1,\mu_2,b,j,t)]<C_{\mu_1,\mu_2,b}
$$
for some constant $C_{\mu_1,\mu_2,b}<\infty$.

%% file: TechnicalLemma.tex
On the interval the $\cJ$ (included in $\cI_j$ by hypothesis), the posterior distribution 
$\pi_{1,s}=\pi_{1,\tau_{j}}$ is fixed and $\theta_{1,s}$ is, when conditioned on $S_{1,\tau_j}$, an i.i.d. sequence with common distribution $\text{Beta}(S_{1,\tau_j} + 1,j - S_{1,\tau_j} +1)$. Hence,
$$\bP\left(\theta_{1,s} \leq \mu_2 + \delta |s \in \cJ \right)  = F_{(S_{1,\tau_j} + 1,j - S_{1,\tau_j} +1)}^{\text{Beta}}(\mu_2+\delta) = 1- F_{(j+1,\mu_2 + \delta)}^B(S_{1,\tau_j})$$
where we use the link between the tail of Beta and Bernoulli distribution mentioned above. Using the independence of the $\theta_{1,s}$  gives
$$\bP\left(\forall s\in J, \  \theta_{1,s} \leq \mu_2 + \delta |S_{1,\tau_j}\right)  = \left(1 -  F_{(j+1,\mu_2 + \delta)}^B(S_{1,\tau_j})\right)^{|J|}\leq \left(1 -  F_{(j+1,\mu_2 + \delta)}^B(S_{1,\tau_j})\right)^{f(t)}$$
Finally
$$
\bP\left(\forall s\in J, \  \theta_{1,s} \leq \mu_2 + \delta \right)  =  \bE\left[\bP\left(\forall s\in J, \  \theta_{1,s} \leq \mu_2 + \delta |S_{1,\tau_j}\right)\right] 
 \leq  \bE\left[\left(1 -  F_{(j+1,\mu_2 + \delta)}^B(S_{1,\tau_j})\right)^{f(t)}\right] 
$$
An exact computation of this expectation leads to
\[
\bE\left[(1-F^B_{(j+1,\mu_2+\delta)}(S_{1,\tau_j}))^{f(t)}\right]  =  \sum_{s=0}^j(1-F^B_{(j+1,\mu_2+\delta)}(s))^{f(t)}f^B_{j,\mu_1}(s)
\]
To simplify notation, from now on let $y=\mu_2+\delta$. Using, as in \cite{Agrawal:Goyal}, that $F_{j+1,y}^B(s)=(1-y)F_{j,y}^B(s) + yF_{j,y}^B(s-1) \geq (1-y)F_{j,y}^B(s)$, we get:
$$(1-F^B_{(j+1,y)}(s))^{f(t)}\leq \exp\left(-f(t)F^B_{(j+1,y)}(s)\right) \leq \exp\left(-f(t)(1-y)F^B_{(j,y)}(s)\right)$$
Therefore,
\[
\bE\left[(1-F^B_{(j+1,\mu_2+\delta)}(S_{1,\tau_j}))^{f(t)}\right]  \leq  \sum_{s=0}^j\exp\left(-f(t)(1-y)F^B_{(j,y)}(s)\right)f^B_{j,\mu_1}(s)
\]
Using the fact that for $s\geq \lceil yj\rceil, F^B_{j,y}(s) \geq \frac{1}{2}$ (since the median of a binomial distribution with parameters $j$ and $y$ 
is $\lceil yj\rceil$ or $\lfloor yj\rfloor$), we get
\begin{align*}
\bE&\left[(1-F^B_{(j+1,\mu_2+\delta)}(S_{1,\tau_j}))^{f(t)}\right] \\
 & \leq  \sum_{s=0}^{\lfloor jy \rfloor} \exp\left(-f(t)(1-y)F^B_{(j,y)}(s)\right)f^B_{j,\mu_1}(s) 
+ \sum_{s=\lceil jy \rceil}^{j} \left(\frac{1}{2}\right)^{(1-y)f(t)}f^B_{j,\mu_1}(s) \\
& \leq \underbrace{\sum_{s=0}^{\lfloor jy \rfloor}\exp\left(-f(t)(1-y)F^B_{(j,y)}(s)\right) f^B_{j,\mu_1}(s)}_{E} + \left(\frac{1}{2}\right)^{(1-y)f(t)}.
\end{align*}
It is easy to show that for every $\lambda>1$,$\forall  x >0,  x^\lambda\exp(- x) \leq \left(\frac{\lambda}{e}\right)^\lambda$ This allows us to upper-bound the exponential
for all $\lambda>1$, using $C_\lambda=\left(\frac{\lambda}{e}\right)^\lambda$,by:
\begin{align*}
(E)  \leq \frac{C_\lambda}{\left(f(t)(1-y)\right)^\lambda} \sum_{s=0}^{\lfloor jy \rfloor}\frac{f^B_{j,\mu_1}(s)}{\left(F^B_{(j,y)}(s)\right)^\lambda} 
\leq
\frac{C_\lambda}{\left(f(t)(1-y)\right)^\lambda} \sum_{s=0}^{\lfloor jy \rfloor}\frac{f^B_{j,\mu_1}(s)}{\left(f^B_{(j,y)}(s)\right)^\lambda}
\end{align*}
Now, inspired by Agrawal and Goyal's work (proof of Lemma 3) we compute:
\begin{align*}
 \frac{f^B_{j,\mu_1}(s)}{\left(f^B_{(j,y)}(s)\right)^\lambda}&=\frac{{j \choose s}\mu_1^s(1-\mu_1)^{j-s}}{{j \choose s}^\lambda(y^\lambda)^s((1-y)^\lambda)^{j-s}} \leq \frac{\mu_1^s(1-\mu_1)^{j-s}}{(y^\lambda)^s((1-y)^\lambda)^{j-s}} \\
& = \left(\frac{1-\mu_1}{(1-y)^\lambda}\right)^j\left(\frac{\mu_1(1-y)^\lambda}{y^\lambda(1-\mu_1)}\right)^s
\end{align*}
Let $R_\lambda(\mu_1,y)=\frac{\mu_1(1-y)^\lambda}{y^\lambda(1-\mu_1)}$. There exists some $\lambda_1>1$ such that, if $\lambda < \lambda_1$, $R_\lambda>1$.
More precisely, 
$$R_\lambda>1 \Leftrightarrow \frac{\mu_1}{1-\mu_1}>\left(\frac{y}{1-y}\right)^\lambda \Leftrightarrow \ln\left(\frac{\mu_1}{1-\mu_1}\right) > \lambda\ln\left(\frac{y}{1-y}\right)$$
and so 
$$ \lambda_1(\mu_1,y) = 
\left\{
\begin{array} {cl}
\frac{\ln\left(\frac{\mu_1}{1-\mu_1}\right)}{\ln\left(\frac{y}{1-y}\right)} & \text{if} \ y > \frac{1}{2}\\
+\infty & \text{if} \ y < \frac{1}{2}
\end{array}
\right.
$$
For $1<\lambda<\lambda_1$:
\begin{align*}
\sum_{s=0}^{\lfloor jy \rfloor}\frac{f^B_{j,\mu_1}(s)}{(f^B_{(j,\mu_2+\delta)}(s))^\lambda} & \leq \left(\frac{1-\mu_1}{(1-y)^\lambda}\right)^j
\sum_{s=0}^{\lfloor jy \rfloor}R_\lambda^s =\left(\frac{1-\mu_1}{(1-y)^\lambda}\right)^j\frac{R_\lambda^{\lfloor jy \rfloor+1}-1}{R_\lambda-1} \\
& \leq \left(\frac{1-\mu_1}{(1-y)^\lambda}\right)^j\frac{R_\lambda}{R_\lambda-1}R_\lambda^{jy}  = \frac{R_\lambda}{R_\lambda-1}\left(\frac{1-\mu_1}{(1-y)^\lambda}\right)^{j-jy}\left(\frac{\mu_1}{y^\lambda}\right)^{jy} \\
& = \frac{R_\lambda}{R_\lambda-1}e^{-jd_\lambda(y,\mu_1)}
\end{align*}
where $d_\lambda(y,\mu_1)=y\ln\left(\frac{y^\lambda}{\mu_1}\right) + (1-y)\ln\left(\frac{(1-y)^\lambda}{1-\mu_1}\right)$. Rearranging we can write
$$d_\lambda(y,\mu_1)  =  \lambda\left[y\ln(y) + (1-y)\ln(1-y)\right] - \left[y\ln(\mu_1) + (1-y)\ln(1-\mu_1)\right]$$
which is an affine function of $\lambda$ with negative slope ($y\ln(y) + (1-y)\ln(1-y)<0$ for all $y\in(0,1)$) and $d_1(y,\mu_1)=K\left(y,\mu_1\right)>0$.
Hence, for fixed $0<y<\mu_1\leq 1$ this function is positive whenever 
$$\lambda < \frac{y\ln(\mu_1) + (1-y)\ln(1-\mu_1)}{y\ln(y) + (1-y)\ln(1-y)}=:\lambda_2(\mu_1,y).$$
Clearly, $\lambda_2(\mu_1,y)>1$ and we choose $\lambda_0=\min(\lambda_1,\lambda_2)$. After some calculation one can show that $\lambda_2\leq \lambda_1$, and therefore that
$$\lambda_0(\mu_1,\mu_2)=\lambda_2(\mu_1,\mu_2+\delta)=1 + \frac{K(\mu_2+\delta,\mu_1)}{(\mu_2+\delta)\ln\frac{1}{\mu_2+\delta} + (1-\mu_2 - \delta)\ln\frac{1}{1-\mu_2 - \delta}}.$$
To obtain the constants used in the statement of the lemma we define $d_{\lambda,\mu_1,\mu_2}:=d_\lambda(y,\mu_1)$
\[
C_{\lambda,\mu_1,\mu_2}:=C_{\lambda_0}(1-\mu_2-\delta)^{-\lambda}\frac{R_\lambda}{1-R_\lambda}.
\]
This concludes the proof.

%% file: Experiments.tex
We illustrate here the performance of Thompson Sampling on numerical experiments with Bernoulli rewards. First we compare in terms of cumulative regret up to horizon 
$T=10000$ Thompson Sampling to UCB, KL-UCB and Bayes-UCB in two different two-arms problem, one with small rewards and the other with high rewards, with different gaps 
between the 
parameters of the arms. Figure \ref{2armsexp} shows Thompson Sampling always outperforms KL-UCB and also Bayes-UCB for large horizons. The three optimal policies 
are significantly better than UCB, even for small horizons.

Figure \ref{10armsexp} displays for several algorithms an estimation of the distribution of the cumulative regret based on $N=50000$ trials, for a horizon 
$T=20000$ in a 10-armed bandit problem with $$\mu=(0.1,0.05,0.05,0.05,0.02,0.02,0.02,0.01,0.01,0.01).$$
The first two algorithms are variants of UCB. Of these the UCB-V algorithm is close to the index policy to which Thompson Sampling is compared in \cite{ChapelleLiEmpirical} in the Bernoulli setting, but this
policy is not known to be optimal. This algorithm incorporates an estimation of the variance of the rewards in the index which is defined to be, for an arm that have produced $k$ rewards in $n$ draws,
$$\frac{k}{n} + \sqrt{\frac{2\log(t)}{n}\frac{k}{n}\left(1-\frac{k}{n}\right)} + \frac{3\log(t)}{n}$$
The other algorithms displayed in Figure \ref{10armsexp} have a mean regret closer (sometimes smaller) than the lower bound (which is only 
asymptotic), and among them, Thompson is the best. It is also the easiest optimal policy to implement, since the optimization problem solved in KL-UCB 
and even the computation of the quantiles in Bayes-UCB are more costly than producing one sample from the posterior for each arm.
\begin{figure}[p]
  \centering
  \begin{minipage}{0.49\textwidth}
  \includegraphics[width=1.1\linewidth]{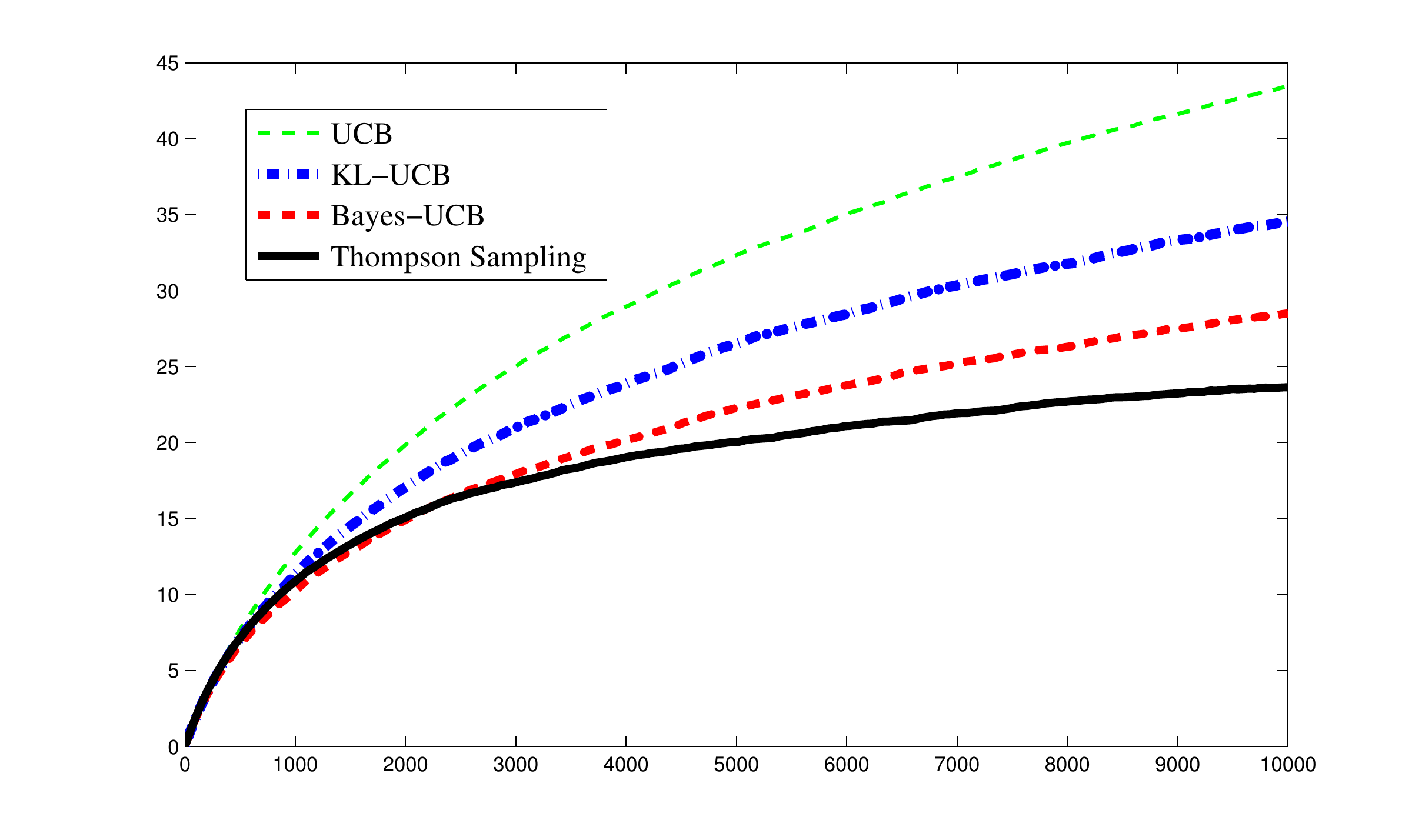}
  \end{minipage}
  \begin{minipage}{0.49\textwidth}
  \includegraphics[width=1.1\linewidth]{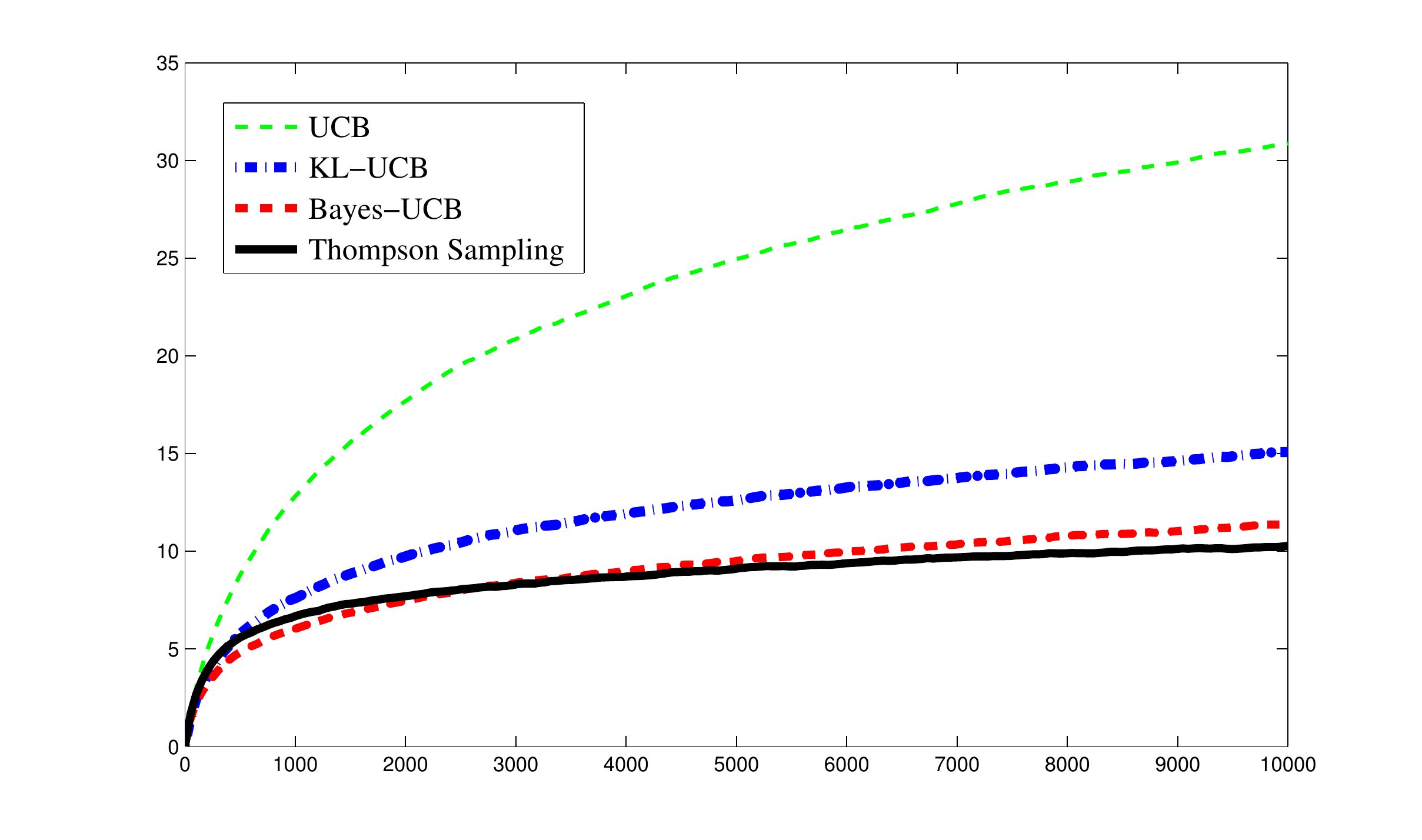}
  \end{minipage}
  \caption{\label{2armsexp} Cumulated regret for the two armed-bandit problem with $\mu_1=0.2, \mu_2 = 0.25$ (left)
    and $\mu_1=0.8, \mu_2=0.9$ (right). Regret is estimated as an average over $N=20000$ trials.}
\end{figure}
\begin{figure}[p]
  \centering
  \includegraphics[width=0.9\linewidth]{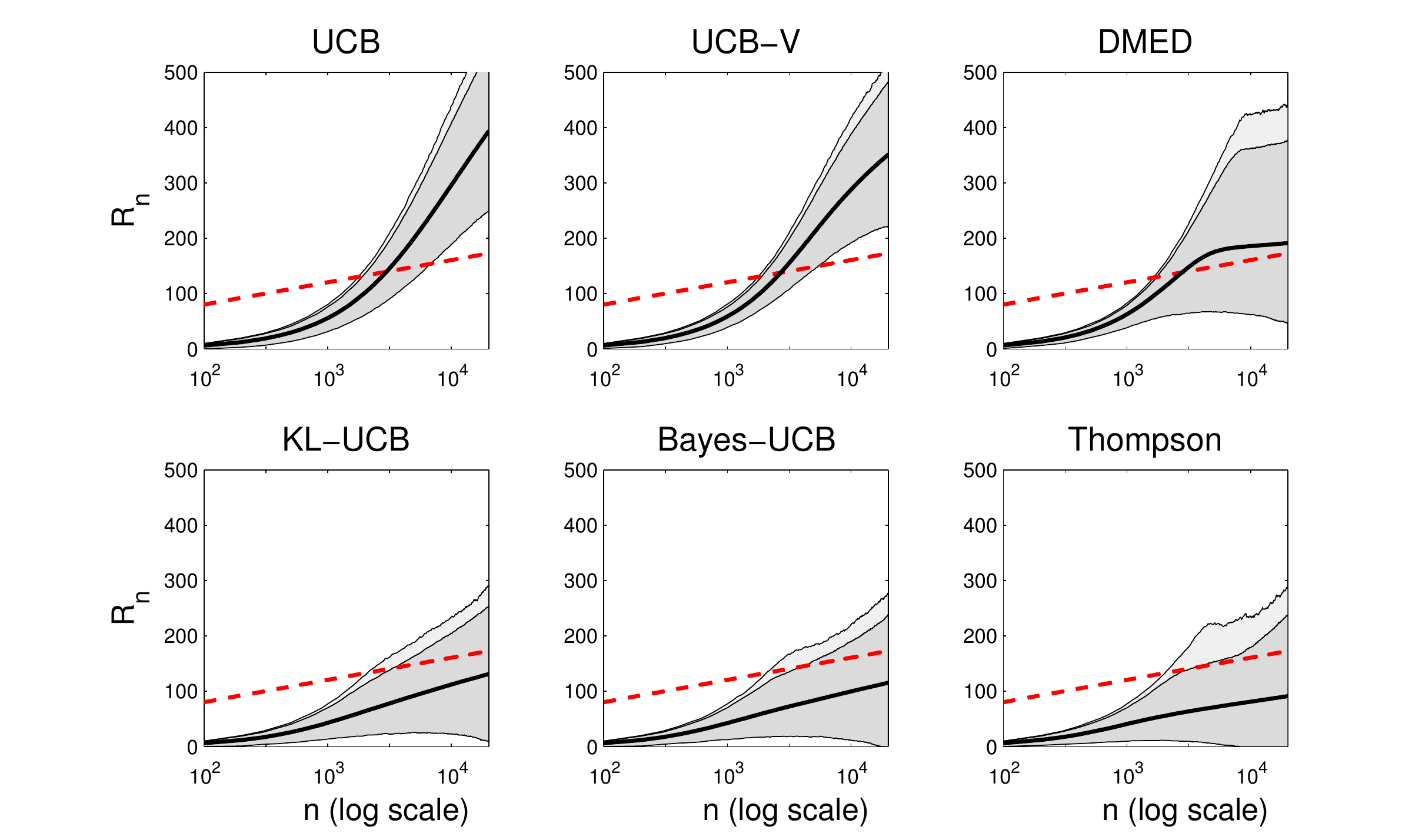}
  \caption{\label{10armsexp} Regret of the various algorithms as a function of time (on a log scale).
 On each graph, the red dashed line shows the lower bound, the solid bold curve corresponds
to the mean regret while the dark and light shaded regions show respectively the central 99\%
and the upper 0.05\%}
\end{figure}

%% file: Conclusion.tex
This paper provides the first proof of the asymptotic optimality of Thompson Sampling for Bernoulli bandits.
Moreover the proof consists in a finite time analysis comparable with that of other known optimal policies.
We also provide here simulations showing that Thompson Sampling outperforms currently known optimal policies.

Our proof of optimality borrows some ideas from Agrawal and Goyal's paper, such as the notion of saturated arms.
However we make use of ideas together with our own to obtain a stronger result, namely control over the tail of $N_{1,t}$ rather than its expectation.
This is a valuable result which justifies the complexity of the proof of Proposition 2. 
Indeed control over these tails allows us to give a simpler finite time analysis for Thompson Sampling 
which is closer to the arguments for UCB-like algorithms, and also yields the optimal asymptotic rate of Lai and Robbins. 

Thanks to the generalisation pointed out in \cite{Agrawal:Goyal}, 
the Bernoulli version of Thompson Sampling can be applied to bandit problems with bounded rewards, 
and is therefore an excellent alternative to UCB policies.
It would also be very natural to generalise Thompson to more complex reward distributions, 
choosing a prior appropriate for the assumptions on these distributions. 
Indeed, even in complex settings where the prior is not computable, Thompson Sampling only requires one sample from the posterior, 
which can be obtained efficiently using MCMC. 
Encouraging numerical experiments for reward distributions in the exponential family using a conjugate prior suggest that 
a generalisation of the proof is achievable. However this poses quite a challenge since the proof here is often heavily dependent on 
specific properties of Beta distributions. 
A natural generalisation would need a prior-dependent finite-time result controlling the 
tail probabilities of posterior distributions as the number of samples increases.

\paragraph{Acknowledgments} We thank Aur\'elien Garivier and Olivier Capp\'e, for many fruitful discussions
and for giving us the opportunity to work together.

This work was supported by the French National Research Agency (ANR-08-COSI-004 project EXPLO-RA) and the European Community’s Seventh Framework Programme (FP7/2007-2013) under grant agreements n\textopenbullet\ 216886 (PASCAL2), and n\textopenbullet\ 270327 (CompLACS).